\documentclass[journal, letter]{IEEEtran}
\usepackage{graphicx} 
\usepackage{amsmath}
\usepackage{amssymb}
\usepackage{amsthm}
\usepackage{xcolor}
\usepackage{soul}
\usepackage{bm}
\usepackage{hyperref}
\usepackage{algorithm}
\usepackage{algorithmic}
\usepackage{mathtools}
\usepackage{diagbox}

\usepackage{tikz}
\usetikzlibrary{positioning}
\usetikzlibrary{decorations.pathreplacing, arrows.meta, calc}

\newtheorem{property}{Property}
\newtheorem{corollary}{Corollary}

\newcommand{\V}{\bm{V}}
\newcommand{\Vi}{\bm{V}^{-1}}
\newcommand{\vvec}{\bm{v}}
\newcommand{\vlie}{\mathfrak{v}}
\newcommand{\W}{\bm{W}}

\newcommand{\dlie}{\mathfrak{d}}

\newcommand{\D}{\bm{D}}

\newcommand{\Exp}{\mathrm{Exp}}
\newcommand{\Log}{\mathrm{Log}}

\DeclareMathOperator*{\argmin}{arg\,min}
\newcommand{\pfrac}[2]{{\frac{\partial #1}{\partial #2}}}

\title{Improved Representation of Matrix Lie Group Operations through Tensor Notation}
\author{Clark N. Taylor\\Air Force Institute of Technology\\WPAFB, OH, USA 45433\\clark.n.taylor@gmail.com}
\date{}

\begin{document}

\maketitle


\begin{abstract}
    Several recent papers have demonstrated the utility of using Lie groups within estimation problems, yielding improved accuracy and consistency.  This paper introduces a new tool for describing operations with matrix Lie groups: tensors and the Einstein summation notation.  While tensors and Einstein notation are well-known in other research fields, applying this mathematical notation to represent and compute matrix Lie derivatives is novel.  More importantly, this new notation greatly clarifies the derivatives and operations necessary to work with matrix Lie Groups in (gradient-based) estimation frameworks.  Therefore, the main contribution of this paper is not a new capability, but a more perspicuous mathematical notation for working with matrix Lie groups.
\end{abstract}


\section{Introduction}
One of the fundamental problems in robotics, space vehicle control, aircraft control, and other fields is estimating rotation, usually in conjunction with other quantities (e.g., location, velocity, etc.)  Because modern estimation frameworks such as the Kalman filter and factor graphs rely on gradients, practical rotation estimation requires the ability to compute derivatives -- both of external quantities with respect to (w.r.t.) the rotation and vice versa.  Rotations, however, are particularly difficult to take derivatives of.  For example, Euler angles are often used to represent rotations, but all have singularities where the derivative is undefined (gimbal lock).  In addition to the singularity problem, there is the more general problem that rotations are typically \emph{applied} within a linear equation using matrix multiplication.  But arbitrary changes to a rotation matrix lead to a matrix that is no longer a valid rotation matrix.  This fundamental conflict between taking derivatives of an equation where rotations are applied and the limited set of matrices that are valid rotation matrices makes derivatives with rotation matrices significantly more complex than derivatives with traditional linear equations.

To overcome this problem, Lie algebras on the rotation manifold can be used.  This enables both the proper computation of derivatives and a better understanding of the uncertainty associated with the computed pose estimates ~\cite{barfoot2014associating,mangelson2020characterizing,begelfor2005put}.  Because of these properties, hundreds of papers have been published which utilize Lie groups in state estimation problems.  While a review of these papers is beyond the scope of this paper, we note several tutorials have been published addressing the mathematics of Lie groups in general\footnote{This paper generally follows the notation introduced in~\cite{sola2018micro} to discuss Lie groups.}~\cite{sola2018micro, hertzberg2013integrating, drummond2003computing}.  

This paper further limits itself to \emph{matrix} Lie groups, though this too has been a field of significant activity.  One particular example that has shown improved estimation performance due to the use of matrix Lie groups is the preintegration of inertial sensors.  Early work used the SO3 Lie group to represent rotations, enabling preintegration of inertial sensors~\cite{forster2016manifold,yang2020analytic}.  Later works extended the preintegration using the $SE_2(3)$ ~\cite{lin2025closed,brossard2021associating,wang2024mavis} and Galilean~\cite{delama2024equivariant} Lie groups, all examples of possible \emph{matrix Lie groups}.

There are two fundamental difficulties to using  derivatives with matrix Lie groups.  The first is a representational difficulty when using the derivative of matrices in more complex equations.  Consider this equation, where the matrix $\V$ is a matrix in a Lie group:
\begin{equation}
    \bm{x} = \bm{AVBy}\label{eq:matrixDerivEq}.
\end{equation}
While the derivative of $\bm{x}$ w.r.t. $\bm{y}$ is straightforward to compute, yielding the Jacobian matrix $\bm{AVB}$, the derivative of $\bm{x}$ w.r.t. $\V$ is not as straightforward. Using a Jacobian matrix implicitly assumes the derivative of a vector w.r.t. a vector, while matrix Lie groups will require derivatives of a matrix w.r.t. a vector (and vice versa).  While the Jacobian of a matrix w.r.t. a vector can be computed by ``vectorizing'' matrices, this destroys the structural relationship of the matrix to other elements in more complex mathematical equations, making other operations more mathematically ``opaque'' to the user.

The second difficulty computing derivatives with Lie groups is the concept that (local) movement on the manifold can be expressed as a vector from the tangent plane, generally of smaller dimension than the number of elements in the matrix Lie group itself.  For example, the SE3 group is typically represented as a $4\times 4$ matrix, while local movement on the manifold can be expressed using a 6-element vector.  Because of this distinction between \emph{how} movement can occur (a vector) and the space in which the movement occurs (a matrix manifold), the typical assumptions for using a Jacobian matrix to represent all derivatives is implicitly broken.  

Note that there are several papers discussing how to compute specific derivatives of matrix Lie groups~\cite{blanco2021tutorial,nurlanovexploring,gallego2015compact}.  Among these papers, however, there is no standard notation for how derivatives are represented and operated with.  In \cite{gallego2015compact}, the value for each $\pfrac{\bm{R}}{\bm{v}_i}$ is given, implying the derivative of a matrix ($\bm{R}$) with respect to each element in the vector $\bm{v}$.  In \cite{blanco2021tutorial} (Equation 10.11), the derivative of a vector w.r.t. a matrix was expressed as:
\[
\frac{\partial \log(\mathbf{R})^{\vee}}{\partial \mathbf{R}} \bigg|_{3 \times 9} = 
\left( \begin{array}{ccc|ccc|ccc}
a_1 & 0 & 0 & 0 & a_1 & b & 0 & -b & a_1 \\
a_2 & 0 & -b & 0 & a_2 & 0 & b & 0 & a_2 \\
a_3 & b & 0 & -b & a_3 & 0 & 0 & 0 & a_3
\end{array} \right) 
\]
with the comment that the 9 components are in column-major order.  Note that the dividing lines in this notation do \emph{not} denote 3 different matrices, but rather a division between rows of a matrix, where each (9-element) row represents a complete matrix.  Such notations can hide the complexity of the matrix/vector derivative relationship or can easily lead to errors in communication of the proper derivatives.  Furthermore, performing operations with these derivatives in larger linear systems is not explicitly defined and must be created in a bespoke manner.

In~\cite{sola2018micro}, derivatives are defined using a circle-plus and circle-minus notation, but the meaning of circle-plus/minus depends on what group you are working with.  While this notation is very concise, the circle-plus/minus in the same equation refer to different groups, making the meaning of the operations implicit, rather than explicit. 


Using tensors and Einstein summation notation to represent these complex derivatives and their interactions with the rest of the system provides a potential solution to these difficulties.  The use of tensors has several advantages, including:
\begin{itemize}
    \item Clarity of notation
    \item Pre-existing computational notation for integration with linear systems
    \item Computational software support
\end{itemize}
Furthermore, we introduce a tensor-based projection operation that simplifies the computation of derivatives needed for matrix Lie groups in an estimation framework.  

The rest of the paper is organized as follows.  In Section~\ref{s:tensors}, we introduce tensors and the Einstein notation for working with tensors.  We also introduce a pseudo-inverse for tensors that will simplify some important operations required to work with matrix Lie groups.  Section~\ref{s:mlgroups} provides a brief review of matrix Lie groups and their key operations.  Section~\ref{s:TandMLGs} discusses how tensor notation can be used to clearly express matrix Lie group operations, particularly in the context of gradient-based optimization algorithms.  Section~\ref{s:results} walks through two examples of using tensors to operate on a matrix Lie group.  The code used to perform the operations can be found at \url{https://github.com/cntaylor/tensor\_matrix\_lie\_group\_example}.  Section~\ref{s:concl} concludes the paper.

\section{Tensor Background}
\label{s:tensors}
A tensor is an $n$-dimensional, regular grid of mathematical elements.  Because tensors can be any dimension, scalars (0-dimensional), vectors (1-dimensional), and matrices (2-dimensional) are all special cases of a tensor.  To reference the individual elements of a tensor, subscripts are often used: $\mathcal{T}_{i,j,k}$ refers to a specific element in a 3D tensor at index locations $i$, $j$, and $k$.

\subsection{Tensor Operations}
In this paper, we are particularly interested in two operations that can be performed with tensors: contraction and element-wise multiplication.  Contraction is performing a dot product across a particular dimension of a tensor.  Consider two tensors of size $4\times 5 \times 3$ ($\mathcal{T}$) and $2\times 5$ ($\mathcal{U}$).  If we contract the second dimension of $\mathcal{T}$ with the second dimension of $\mathcal{U}$, we have a an output $4\times 3 \times 2$ tensor $\mathcal{V}$ defined as follows:

\begin{equation}
\mathcal{V}_{i,k,\ell} = \sum_{j=1}^5 \mathcal{T}_{i,j,k}\mathcal{U}_{\ell,j}\quad\forall i,k,\ell
\end{equation}

Note that matrix multiplication is a tensor contraction of two, 2D tensors.  In matrix multiplication, however, which dimensions of the tensors are being contracted are implicitly defined by the placement of the matrices next to each other.  Tensor contraction is a more general operation wherein the dimensions on which the dot product is performed  can be arbitrarily specified.  (This concept is expanded on in Table~\ref{tab:matrixTensorNotation} in the section on Einstein notation.)

Element-wise multiplication can also be performed across similar-sized dimensions in a tensor. Using the same two tensors $\mathcal{T}$ and $\mathcal{U}$ described above, performing an element-wise multiplication across the 2nd dimension of $\mathcal{T}$ and $\mathcal{U}$ leads to a 4D output tensor $\mathcal{W}$:

\begin{equation}
    \mathcal{W}_{i,j,k,\ell} = \mathcal{T}_{i,j,k}\cdot \mathcal{U}_{\ell,j} \quad\forall i,j,k,\ell
\end{equation}

Note that for both contraction and element-wise multiplication to occur, the dimensions being operated on in the two tensors must be the same size.  Also note that for two given input tensors, which dimensions \emph{can} have a contraction or element-wise multiplication occur is not unique.  Consider two tensors $\mathcal{T}$ and $\mathcal{U}$ of sizes $4\times 3 \times 3$ and $3 \times 4$, respectively.  Table~\ref{tab:valid_operators} shows which pairing of dimensions are valid for the contraction and element-wise multiplications, and which are not.

\begin{table}[]
    \caption{This table shows which combinations of dimensions are valid for performing tensor operations on.  A check-mark ($\checkmark$) means an operation could be performed, while an '$\times$' means an operation would be invalid.}
    \label{tab:valid_operators}
    \centering
    \begin{tabular}{|l|c|c|}
        \hline
        \diagbox[width=3.8cm]{$\mathcal{T}$ ($4\times 3 \times 3$)}{$\mathcal{U}$ ($3 \times 4$)} & dim 1 (size 3) & dim 2 (size 4) \\\hline
        dim 1 (size 4) & $\times$       & $\checkmark$   \\\hline
        dim 2 (size 3) & $\checkmark$   & $\times$       \\\hline
        dim 3 (size 3) & $\checkmark$   & $\times$       \\\hline
    \end{tabular}
\end{table}

\subsubsection{Contraction Operation Utility}
There are two particular uses of the contraction operation that are worth highlighting.  First, note that matrix multiplication can be represented as a contraction.  Consider two matrices (2D tensors -- $\mathcal{T}$ and $\mathcal{U}$), of size $n\times m$ and $o\times p$.  If $m$ and $o$ are equal, then the tensor contraction operation
$$
\mathcal{V}_{i,k} = \sum_{j=1}^m \mathcal{T}_{i,j} \mathcal{U}_{j,k}\quad \forall i,k
$$
is the same as matrix multiplication.  

Second, note that when working with derivatives over multiple variables, the chain rule involves a dot-product or contraction.  For example, consider two functions: $f$ and $g$ defined below, and the composite function which chains them together.
\begin{equation}
    \begin{aligned}
    f(u,v) &\rightarrow y\\
    g(x) &\rightarrow u,v\\
    h(x) &= y = f(g(x))\\
    \pfrac{h(x)}{x} &= \pfrac{f(x)}{u}\pfrac{g(x)[1]}{x} + \pfrac{f(x)}{v}\pfrac{g(x)[2]}{x}
    \end{aligned}
\end{equation}
The dot-product of the two partial derivatives yields the total derivative required as shown in the last equation, where $g(x)[1]$ represents the first output of the $g(x)$ function.  Note that the contraction operation with tensors extends this ability to combine partial derivatives in a very flexible and programmatic way.

\subsection{Einstein Notation}
Rather than writing out summations and explicitly denoting which indices the output will be over using the $\forall$ symbol, \emph{Einstein notation} can be used to succinctly describe operations with tensors.  Consider Table~\ref{tab:einstein_example} which shows all of the possible operations (including both contraction and element-wise multiplication) listed in Table~\ref{tab:valid_operators}.  Note how the basic structure of each operation is the same $\mathcal{V}=\mathcal{T}\mathcal{U}$, but the indices on each tensor are used to determine what operation is performed.  In this table, $\mathcal{T}$ was always labeled with indices $i,j,k$, but the actual labels are not important.  Only the ordering of the labels and how those labels correspond with the labels on other tensors determines the operation to be performed.
\begin{table}[]
    \caption{Corresponding with Table~\ref{tab:valid_operators}, this table shows the Einstein summation notation for all possible combinations of $\mathcal{T}$ ($4\times 3 \times 3$) and $\mathcal{U}$($4\times 3$).  Note how the indices implicitly denote what operations occur without using an explicit summation.}
    \label{tab:einstein_example}
    \centering
    \begin{tabular}{|l|c|c|}
    \hline
    &\textbf{Contraction}&\textbf{Element-wise}\\\hline
    $\mathcal{T}$ (2nd dim) \& $\mathcal{U}$ (1st dim)  & $\mathcal{V}_{ik\ell} = \mathcal{T}_{ijk}\mathcal{U}_{j\ell} $ & $\mathcal{V}_{ijk\ell} = \mathcal{T}_{ijk}\mathcal{U}_{j\ell}$\\\hline
    $\mathcal{T}$ (3rd dim) \& $\mathcal{U}$ (1st dim)  & $\mathcal{V}_{ij\ell} = \mathcal{T}_{ijk}\mathcal{U}_{k\ell}$ & $\mathcal{V}_{ijk\ell} = \mathcal{T}_{ijk}\mathcal{U}_{k\ell}$ \\\hline
    $\mathcal{T}$ (1st dim) \& $\mathcal{U}$ (2nd dim)  & $\mathcal{V}_{jk\ell} = \mathcal{T}_{ijk}\mathcal{U}_{\ell i}$ & $\mathcal{V}_{ijk\ell} = \mathcal{T}_{ijk}\mathcal{U}_{\ell i}$ \\\hline
    \end{tabular}
\end{table}

To further illustrate the usage of Einstein notation, Table~\ref{tab:matrixTensorNotation} shows how the product of two similar sized matrices ($\bm{A}$ and $\bm{B}$) can be expressed in traditional matrix notation and in Einstein notation, placing $\bm{A}$ and $\bm{B}$ first.  Note that through the rest of this paper, ``Einstein notation'' and ``tensor notation'' will be used interchangeably.

\begin{table}[]
    \caption{A comparison of different matrix multiplication operations and their expression in Einstein notation.  The first column uses traditional matrix multiplication notation, the second column uses tensor notation with $\bm{A}$ placed first, and the third column with $\bm{B}$ first.}
    \label{tab:matrixTensorNotation}
    \centering
    \begin{tabular}{|c|c|c|}
        \hline
        &\multicolumn{2}{c|}{\textbf{Tensor Notation}}\\
        \textbf{Matrix Notation} & \textbf{$\bm{A}$ first} & \textbf{ $\bm{B}$ first} \\\hline
        $\bm{C} = \bm{AB}$ & $\bm{C}_{ik} = \bm{A}_{ij}\bm{B}_{jk}$ & $\bm{C}_{ik} = \bm{B}_{jk}\bm{A}_{ij}$\\\hline
        $\bm{C} = \bm{BA}$ & $\bm{C}_{ik}= \bm{A}_{jk}\bm{B}_{ij}$ & $\bm{C}_{ik}= \bm{B}_{ij}\bm{A}_{jk}$ \\\hline
        $\bm{C} = \bm{A}^\top\bm{B}$ & $\bm{C}_{ik}= \bm{A}_{ji}\bm{B}_{jk}$ & $\bm{C}_{ik}= \bm{B}_{jk}\bm{A}_{ji}$\\\hline
        $\bm{C} = \bm{B}^\top\bm{A}$ & $\bm{C}_{ik}= \bm{A}_{jk}\bm{B}_{ji}$ & $\bm{C}_{ik}= \bm{B}_{ji}\bm{A}_{jk}$ \\\hline
    \end{tabular}
\end{table}

In the remainder of this paper, the utility of this notation for defining derivatives on a Lie Group should become apparent.

\subsubsection{Computational implementations of Einstein notation}
In addition to explicitly describing tensor contraction in mathematical formulas, several computational tools are available in modern programming languages to support Einstein notation.  In Python, the numpy library implements the \texttt{einsum} function, where the indices to be operated on are specified as a string, followed by the tensors themselves.  Examples of using the \texttt{einsum} function to optimize over rotation matrices are included in the github repository associated with this paper\footnote{\url{https://github.com/cntaylor/tensor\_matrix\_lie\_group\_example}}. In addition to Python, einsum implementations are available in C++ (using the ``Einsums'' library\cite{einsums}), Rust, R, and Julia.

\subsection{Tensors and derivatives of matrices}
As mentioned in the introduction, one of the difficulties of working with a matrix Lie group is that the derivative of matrices are difficult to handle in a regular matrix-based linear equation. The Kronecker product and vectorization of the matrix can be used to compute the derivatives, but it leads to an equation that differs substantially in form from the original equation. Tensors, on the other hand, make the derivatives of matrices simple to represent using the following property.
\begin{property}
\label{prop:simpleMatrixDeriv}
Consider the tensor equation:
\begin{equation}
    \mathcal{A}_{\mathcal{IJL}} = \mathcal{B}_{\mathcal{IKL}}\mathcal{C}_{\mathcal{JKL}}
\end{equation}
where $\mathcal{I}$ represents a set of indices unique to tensor $\mathcal{B}$, $\mathcal{J}$ a set of indices unique to $\mathcal{C}$, $\mathcal{K}$ a set of indices across which contraction is performed, and $\mathcal{L}$ a set of indices across which element-wise multiplication will occur.  If the partial derivatives of this equation are desired w.r.t. another tensor $\mathcal{D}$ with indices $\mathcal{M}$, this can be represented as:
\begin{equation}
    \pfrac{\mathcal{A}_{\mathcal{IJL}}}{\mathcal{D}_\mathcal{M}} = 
    \pfrac{\mathcal{B}_\mathcal{IJL}}{\mathcal{D}_\mathcal{M}}\mathcal{C}_\mathcal{JKL} + \mathcal{B}_\mathcal{IJL}\pfrac{\mathcal{C}_\mathcal{JKL}}{\mathcal{D}_\mathcal{M}}
\end{equation}
\end{property}
\begin{proof}
Consider the equation for an individual element of $\mathcal{A}$, where $\bar{\mathcal{I}}$ represents a value for all the indices within index set $\mathcal{I}$.  In this case, $\mathcal{A}_{\bar{\mathcal{I}},\bar{\mathcal{J}},\bar{\mathcal{L}}}$ represents a scalar value, and the generic matrix equation above can be re-written as:
\begin{equation}
\mathcal{A}_{\bar{\mathcal{I}},\bar{\mathcal{J}},\bar{\mathcal{L}}} = \sum_{\bar{\mathcal{K}}=\mathcal{K}} \mathcal{B}_{\bar{\mathcal{I}}\bar{\mathcal{K}}\bar{\mathcal{L}}}\mathcal{C}_{\bar{\mathcal{J}}\bar{\mathcal{K}}\bar{\mathcal{L}}} \quad \forall \bar{\mathcal{I}}\bar{\mathcal{J}}\bar{\mathcal{L}}
\end{equation}
Because both $\mathcal{B}_{\bar{\mathcal{I}}\bar{\mathcal{K}}\bar{\mathcal{L}}}$ and $\mathcal{C}_{\bar{\mathcal{J}}\bar{\mathcal{K}}\bar{\mathcal{L}}}$ are scalars, the derivative with respect to specific (scalar) values within $\mathcal{D}$ will be:
\begin{equation}
\pfrac{\mathcal{A}_{\bar{\mathcal{I}}\bar{\mathcal{J}}\bar{\mathcal{L}}}}{\mathcal{D}_{\bar{\mathcal{M}}}} = \sum_{\bar{\mathcal{K}}=\mathcal{K}} \pfrac{\mathcal{B}_{\bar{\mathcal{I}}\bar{\mathcal{K}}\bar{\mathcal{L}}}}{\mathcal{D}_{\bar{\mathcal{M}}}}
\mathcal{C}_{\bar{\mathcal{J}}\bar{\mathcal{K}}\bar{\mathcal{L}}}
+ \mathcal{B}_{\bar{\mathcal{I}}\bar{\mathcal{K}}\bar{\mathcal{L}}}
\pfrac{\mathcal{C}_{\bar{\mathcal{J}}\bar{\mathcal{K}}\bar{\mathcal{L}}}}{\mathcal{D}_{\bar{\mathcal{M}}}}
\quad \forall \bar{\mathcal{I}}\bar{\mathcal{J}}\bar{\mathcal{L}}\bar{\mathcal{M}}
\end{equation}
\end{proof}

\begin{corollary}
If only one tensor ($\mathcal{B}$, without loss of generality) is a function of $\mathcal{D}$, then the derivative from Property~\ref{prop:simpleMatrixDeriv} can be written as
\begin{equation}
    \pfrac{\mathcal{A}_{\mathcal{IJL}}}{\mathcal{D}_\mathcal{M}} = \pfrac{\mathcal{B}_\mathcal{IJL}}{\mathcal{D}_\mathcal{M}} \mathcal{C}_{\mathcal{JKL}}
\end{equation}
\label{c:singleDeriv}
\end{corollary}
\begin{proof}
    This corollary is a straight-forward extension of Property~\ref{prop:simpleMatrixDeriv}.
\end{proof}

\subsection{Tensors and Linear Subspaces of matrices}
\label{ss:projection}
To appropriately work with Lie Algebras (defined in the next section), we discuss two operations that can be used when working with linear subspaces of matrices: synthesis and projection.  These subspaces are the set of matrices that are expressed as a linear combination of a set of basis matrices.  Examples of linear matrix subspaces include skew symmetric and symmetric matrices, with the basis for $3\times 3$ and $2\times 2$ matrices, respectively, as:
\begin{center}
\begin{tabular}{ccc}
Basis matrix & Skew Symmetric & Symmetric\\
    $\bm{B}_1$ & $\begin{bmatrix}
        0 & 1 & 0\\
        -1 & 0 & 0 \\
        0 & 0 & 0
    \end{bmatrix}$ & 
    $\begin{bmatrix}
        1 & 0 \\ 0 & 0
    \end{bmatrix}$ \\[1.5em]
    $\bm{B}_2$ & $\begin{bmatrix}
        0 & 0 & 1\\
        0 & 0 & 0 \\
        -1 & 0 & 0
    \end{bmatrix}$ & 
    $\begin{bmatrix}
        0 & 0 \\ 0 & 1
    \end{bmatrix}$ \\[1.5em]
    $\bm{B}_3$ & $\begin{bmatrix}
        0 & 0 & 0\\
        0 & 0 & 1 \\
        0 & -1 & 0
    \end{bmatrix}$ & 
    $\begin{bmatrix}
        0 & 1 \\ 1 & 0
    \end{bmatrix}$     
\end{tabular}
\end{center}

\textbf{Synthesis:}
Having defined the basis set, any matrix $\V$ in the linear subspace of matrices can be expressed as a vector ($\vvec$), where the elements of $\vvec$ denote the coefficients used to form the matrix in the subspace as follows:
\begin{equation}
    \V= \sum_i \vvec_i\bm{B}_{i}
\end{equation}

Using tensor notation, we can denote the basis set of matrices as $\mathcal{B}_{ijk}$, where the first index is an index into the basis set of matrices, and the 2nd and 3rd indices are the row and column indices of the matrices, respectively. Given a vector $\vvec$, we can obtain the matrix $\V$ using the contraction operation: $\V_{jk} = \vvec_i\mathcal{B}_{ijk}$.

\textbf{Projection:} Given a generic matrix $\bm{M}$, we would like to find the vector $\vvec^*$ that yields the closest (in a least squares sense) matrix in the spanning set of $\mathcal{B}$.  More formally:
\begin{equation}
\vvec^* = \argmin_{\vvec} \sum_j \sum_k (\bm{M}_{jk} - \vvec_i\mathcal{B}_{ijk})^2
\end{equation}
If $\mathcal{B}$ forms an orthonormal basis, i.e.
\begin{equation}
    \mathcal{B}_{ikl}\mathcal{B}_{jkl} = \begin{cases}
        1 & \text{if }i=j\\
        0 & \text{if }i\ne j
    \end{cases}
\end{equation}
then $\vvec_i^* = \mathcal{B}_{ijk}\bm{M}_{jk}$.  If $\mathcal{B}$ is not orthonormal, then a projection matrix $\bm{P}$ can be defined as:
\begin{equation}
\bm{P}_{ij} \triangleq \mathcal{B}_{ikl}\mathcal{B}_{jkl}
\label{eq:projMatrix}
\end{equation}
The projection of a matrix $\bm{M}$ on the basis set of matrices can then be found using the inverse of the projection matrix as:
\begin{equation}
    \vvec^*_i = (\bm{P}^{-1})_{ij}\mathcal{B}_{jkl}\bm{M}_{kl}
    \label{eq:matrixProjection}
\end{equation}
Note that this projection operation is equivalent to a Moore-Penrose pseudo-inverse if the basis matrices and input matrix were all vectorized.

\section{Matrix Lie Groups}
\label{s:mlgroups}
In this section, we briefly review matrix Lie groups and the key operations involved in using them in a gradient-based optimization routine.  We do not introduce tensor notation in this section, just the basic operations required to work with a matrix Lie group.  To properly describe a \emph{matrix Lie group}, the following subsections address each word in detail (in reverse order).

\subsection{\emph{Group}}
A \emph{group} ($\mathbb{G}$) is a set of objects with an operator ($\circ$) that when used on two elements of the group yields another element of the group (\emph{closure}).  More formally:
\begin{equation}
    \V_3 = \V_1 \circ \V_2\ \Rightarrow\ \V_3 \in \mathbb{G}\quad \forall \V_1, \V_2 \in \mathbb{G}
\end{equation}
Furthermore, a group must have an identity element ($\bm{I}$) and an inverse for every element ($\Vi$) such that:
\begin{equation}
\begin{aligned}
    \V \circ \bm{I} &= \V \quad \forall \V\\
    \V \circ \Vi &= \bm{I} \quad \forall \V
\end{aligned}
\end{equation}

\subsection{\emph{Lie} Group}
A Lie (pronounced L\={e} or ``lee'') group is a group that is also a differential manifold, implying that local movements on the manifold can be represented as a Euclidean space and that derivatives of that movement always exist.  Intuitively, a manifold can be thought of as a lower dimensional object (dimension $m$) placed in a higher dimensional space (dimension $n$, $m \le n$).  For example, lines are one-dimensional objects in higher dimensional spaces, defined by a specific vector.  Multiple vectors define a hyper-plane, defining a linear manifold where $m$ is the number of vectors spanning the hyper-plane, and the size of each vector is $n$.  More complex examples of a manifold include a hollow sphere, which is a 2D object ($m=2$) in a 3D space ($n=3)$\footnote{Note that while a 2D sphere in 3-space \emph{is} a manifold, it is \emph{not} a Lie group as there is no operator to combine elements of a sphere to form a new element of the sphere -- i.e., it is a manifold but not a group.}; rotation quaternions, which are a 3D object in a 4D space; and rotation matrices (SO3), which are a 3D object (the axes of rotation) in a 9D space (the 9 elements of the $3\times 3$ matrix.)  Each of these examples are also smooth, or differential manifolds as they are locally equivalent to a constant-dimensional hyper-plane at all points.  The set of integers, on the other hand, are a group but not a manifold as there is no \emph{differential} way to move from one integer to the next.

Associated with any Lie group is a Lie algebra, which is the tangent (hyper-)plane associated with the origin of the Lie group\footnote{Following previous Lie group papers, we denote elements of the Lie algebra using Fraktur letters (i.e., $\dlie$).}  To understand the importance of the Lie algebra, consider the case where you have an element of the Lie Group ($\V$) and want to characterize (continuous) movements about that element within the Group.  The simplest movement is no movement and can be expressed as $\V \circ \bm{I}$, corresponding to an all-zero Lie algebra vector.  To represent infinitesimal movements, a Taylor series expansion combining the identity and the Lie algebra together -- $\D = \bm{I}+\dlie$ -- is used, leading to infinitesimal movements $\W \circ \V$.  For non-infinitestimal movements, a differential equation is solved where the movement at each point is an element of the Lie algebra, leading to an \emph{exponential} operator that maps elements of the Lie algebra to elements of the Lie group.  Therefore, any continuous movement from $\V$ within the group can be represented as
\begin{equation}
    \V \circ \exp(\dlie)
\end{equation}
The $\log$ operator maps the opposite direction, from continuous movements in the Lie group to elements of the Lie algebra.

Because a Lie manifold is an $m$-dimensional object in an $n$-dimensional space with $m\le n$, a tangent plane to the manifold will be an $m-$dimensional hyperplane spanned by a set of $m$ basis vectors.  Therefore, any element $\vlie$ of the Lie algebra can also be expressed as an $m$ dimensional vector $\vvec$, representing a set of scalar multiples to the basis vectors of the Lie algebra. The mapping from a vector ($\vvec$) to a member of the Lie algebra is denoted as $\vlie = \vvec^\wedge$, and in the opposite direction as $\vvec = \vlie^\vee$.

By combining the exponentiation operator with the mapping of an $m$-dimensional vector to an $n$-dimensional tangent vector, a complete operation from a $m$-dimensional vector to an element of the group, $\V$, is defined.  We label this combined operation $\Exp()$ (with a capital letter) to distinguish it from a traditional exponential operator.  The inverse of this operation is labeled $\Log()$.  See Fig.~\ref{fig:lie_exp_log_relationship}(a) for a graphical representation of these operations.  Note that the \emph{hat} operator ($\vvec^\wedge$) that maps between an $m$-dimensional vector and $n$-dimensional Lie algebra is a synthesis operator that takes a lower dimensional vector and, using a basis set, creates a higher dimension ($n$) vector.  The $\emph{vee}$ operator ($\vvec^\vee$), on the other hand, can be represented as a projection operation, though the mapping should be exact if $\vlie$ is in the Lie algebra.  \emph{The ability to map a low dimensional vector to continuous movements on the manifold is what makes Lie Groups a powerful mathematical tool.}

\subsubsection{Lie Group Example}

An example of the complete mapping from vector to local Lie group movements is shown in \ref{fig:lie_exp_log_relationship}(b).  The Lie group in this case is the unit circle in the complex plane, with the complex multiply as the group operator.  Note that multiplying two elements on the unit circle together yields another element on the unit circle (closure), an identity element exists ($1+0i$) and the inverse of any element is defined -- $(a+bi)^{-1} = a-bi$.

At the identity element of the unit circle ($1+0i$), the tangent plane is a vertical line, or the set of vectors $0+\theta i,\ \theta\in\mathbb{R}$.  This is the Lie algebra for this group.  Because the plane is characterized by a single value, $\theta$, the hat operator is $\theta^\wedge = 0 + \theta i$, and the vee operator maps $0 + \theta i \to \theta$.  The exponential operator turns a member of the Lie algebra into a group member using Euler's identity: $e^{0+\theta i} = \cos\theta + i\sin\theta$, and $\Exp$ takes a mapping from $\theta$ all the way to a member of the Lie group ($\Exp(\theta) = \cos\theta+i\sin\theta$).

With these operations defined, any local movement around an element of the group ($\V$) can be defined by a single scalar.  Shown in in Fig.~\ref{fig:lie_exp_log_relationship}(b) is a movement of $d =0.8$ from $\V$.  Note that 0.8 maps to an element of the Lie algebra ($0 + 0.8i$) which is then exponentiated to a point in the Lie Group ($\bm{D}$).  By combining $\bm{V}$ and $\bm{D}$ together, a local movement from $\bm{V}$ is found.  Therefore, the value $d$ provides a mapping from a scalar value to a movement on the unit circle in the complex plane.  While this is a relatively simple example, it demonstrates the power of Lie group theory to represent local movements about a Lie group member with low-dimensional vectors (in this case, a scalar.) 

\begin{figure*}
    \centering
\begin{tabular}{cc}
\resizebox{.44\textwidth}{!}{
\begin{tikzpicture}[
    node distance=3.5cm, 
    main_node/.style={circle, draw, minimum size=1.2cm, font=\large},
    arrow_label/.style={font=\small, midway},
    column_label/.style={font=\Large\bfseries, text=gray!60}
]

    \node[column_label] at (0, 3) {$\mathbb{R}^m$};
    \node[column_label] at (3.5,3) {$\mathfrak{g}$};
    \node[column_label] at (7, 3) {$\mathbb{G}$};

    \node (v) at (0,0) {$\vvec$};
    \node (vlie) at (3.5,0) {$\vlie$};
    \node (V) at (7,0) {$\V$};

    
    \draw[->, thick, >=Stealth] (v.north east) to [out=45, in=135] 
        node[arrow_label, above] {$\vvec^\wedge$} (vlie.north west);
    \draw[->, thick, >=Stealth] (vlie.north east) to [out=45, in=135] 
        node[arrow_label, above] {$\exp(\vlie)$} (V.north west);

    \draw[->, thick, >=Stealth] (V.south west) to [out=225, in=315] 
        node[arrow_label, below] {$\log(\V)$} (vlie.south east);
    \draw[->, thick, >=Stealth] (vlie.south west) to [out=225, in=315] 
        node[arrow_label, below] {$\vlie^\vee$} (v.south east);

    
    \draw[->, ultra thick, green!60!black, >=Stealth] (v.north) .. controls +(1, 2.0) and +(-1, 2.0) .. 
        node[midway, above=0.1cm, font=\bfseries] {$\text{Exp}(\vvec)$} (V.north);

    \draw[->, ultra thick, red!70!black, >=Stealth] (V.south) .. controls +(-1, -2.0) and +(1, -2.0) .. 
        node[midway, below=0.2cm, font=\bfseries] {$\text{Log}(\V)$} (v.south);

\end{tikzpicture}} & 
\resizebox{.52\textwidth}{!}{
\begin{tikzpicture}[scale=2.5, >=Stealth]

    \draw[thin, gray!20] (-1.5,-1.5) grid (2.5,1.5);
    \draw[->, thick] (-1.2,0) -- (2.2,0) node[right] {$\text{Re}(z)$};
    \draw[->, thick] (0,-1.2) -- (0,1.2) node[above] {$\text{Im}(z)$};

    \draw[very thick, blue!70!black] (0,0) circle (1);

    \filldraw[black] (1,0) circle (0.03);
    \node[below right] at (1,0) {$\mathbf{e} = 1 + 0i$};

    \draw[dashed, red!70!black] (1,-1.2) -- (1,1.5);
    \node[red!80!black, right] at (1, 1.3) {$\mathfrak{g} \cong 0 + \theta i$ (Lie Algebra)};

    \def\wVal{0.8}
    \def\wDeg{\wVal*180/pi)}
    \coordinate (w_origin) at (1,0);
    \coordinate (w_end) at (1,\wVal);
    \draw[->, line width=1.5pt, green!60!black, align=left] (w_origin) -- (w_end) node[midway, right] {$d = {\wVal}$\\[3pt] $\mathfrak{d} = d^\wedge = 0 + {\wVal} i$};
    
    \coordinate (exp_pt) at ({cos(\wDeg)}, {sin(\wDeg)});
    \draw[->, thick, orange] (w_end) to[out=170, in=30] (exp_pt);
    \node[orange, above] at ($(w_end)!0.5!(exp_pt) + (0,0.05)$) {$\exp(\mathfrak{d})$};
    \filldraw[black] (exp_pt) circle (0.03);
    \node[below left] at (exp_pt) {$\bm{D}$};

    \def\angleG{130}
    \coordinate (V) at ({\angleG}:1);
    \filldraw[black] (V) circle (0.03);
    \node[above left] at (V) {$\V$};

    \def\angleFinal{\angleG + \wDeg}
    \coordinate (comp_coord) at ({\angleFinal}:1);

    \draw[->, line width=1.5pt, green!60!black] (\angleG:1) arc (\angleG:\angleFinal:1);
    \filldraw[black] (comp_coord) circle (0.03);
    \node[left] at (comp_coord) {$\V\circ\bm{D}$};
    
    \node[draw, fill=white, font=\small, align=left] at (1.5, -1) {
        \textbf{Key Components:}\\
        1. \textcolor{blue}{Manifold}: $\mathbb{G}: z \in \mathbb{C}, |z|=1$\\
        2.  Origin : $\bm{e}$\\
        3. \textcolor{red!70!black}{Lie Algebra}: $0 + i\theta$ (tangent at $\bm{e}$)\\
        4. \textcolor{orange}{Exp Map}: $\theta \mapsto e^{i\theta}$\\
        5. \textcolor{green!60!black}{Local movement at $\V$}: $\V \circ \Exp(d)$
    };

\end{tikzpicture}
} \\
(a) Lie group operations summary & (b) Operations applied to U(1)
\end{tabular}
\caption{An illustration of the $\Exp$ and $\Log$ operators and their relationship to other operators in Lie groups (a), with a specific example of applying these operations on the unit circle in the complex plane, a Lie group (b).}
    \label{fig:lie_exp_log_relationship}
\end{figure*}
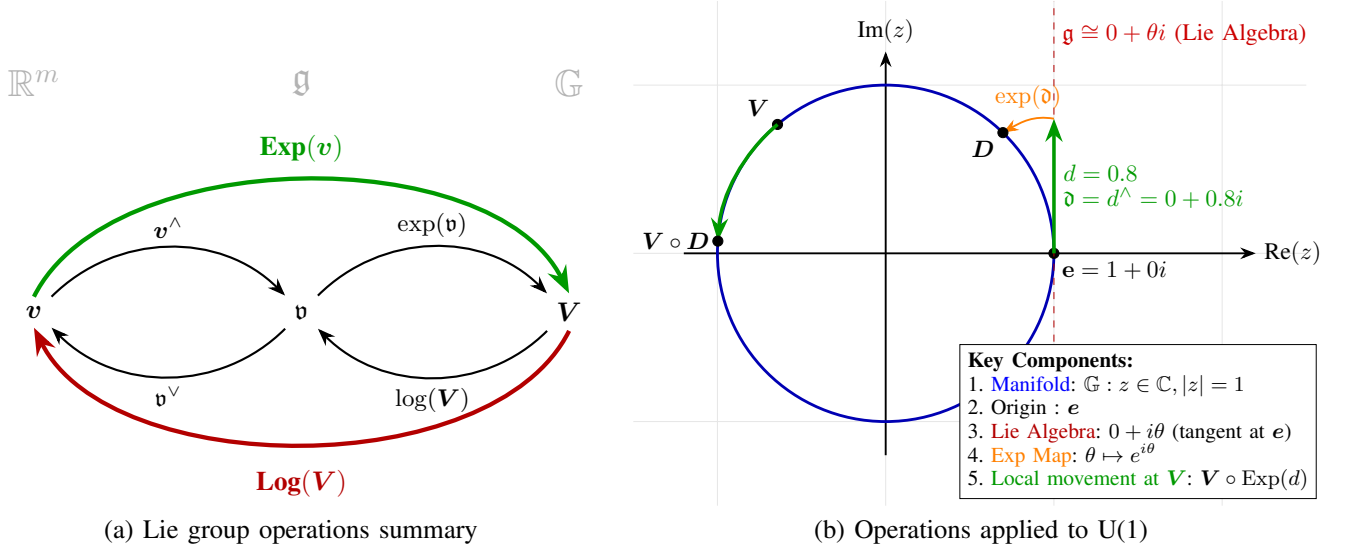

\subsection{\emph{Matrix} Lie Groups}
\label{ss:mlgroups}
Matrix Lie groups are Lie groups that consist of matrices.  The broadest matrix Lie group is the general linear group of a specific size ($\mathrm{GL}(n)$), which consists of all $n\times n$ invertible matrices.  Subsets of this group that are of interest include the special linear group $\mathrm{SL}(n)$ which includes all matrices of determinant of +1 (used to optimize homography matrices for computer vision\cite{begelfor2005put}), the special orthogonal group ($\mathrm{SO}(n)$), with $\mathrm{SO}(3)$ being of particular interest as it represents physically realizable rotations in the 3D world, and the Special Euclidean group ($\mathrm{SE}(n)$), which represents Euclidean movement in the $n$ dimensional space.  $\mathrm{SE}(3)$ is represented by a $4\times 4$ matrix and are the homogeneous matrices used in computer vision and graphics to represent rotation and translation in the 3D world.

Matrix Lie groups are nicely defined by well-known matrix operations.  The group operator ($\circ$) is matrix multiplication, the identity element is the identity matrix, and the inverse is the matrix inverse.  The Lie algebra elements ($\vlie$) are not truly vectors but rather matrices.  The exponential operation on matrices (and its inverse, the logarithm) is well-studied and is available in many computational libraries.  This operation is defined by:
\begin{equation}
    \exp(\vlie) = \bm{I} + \vlie + \frac{\vlie^2}{2!} + \frac{\vlie^3}{3!} + \cdots = \sum_{k=0}^{\infty} \frac{\vlie^k}{k!}.
    \label{eq:exponentiation}
\end{equation}
Note that for small $\vlie$, this exponentiation operator can be approximated by $\bm{I} + \vlie$.  

\section{Tensor Expressions and Matrix Lie Groups}
\label{s:TandMLGs}
In the previous two sections, we have introduced both tensors with some fundamental tensor operations and matrix Lie groups.  In this section, we show how using tensor notation creates a perspicuous representation of matrix Lie group operations that enables several of the fundamental operations needed to work with matrix Lie groups.  In the sections that follow, we introduce formal tensor notation for the \emph{hat} and \emph{vee} operations that map from a Euclidean space to the Lie algebra space (and vice versa).  We also introduce notation for handling derivatives with matrix Lie groups, discussing two types of derivatives that are needed for estimation problems.

\subsection{Mapping vectors to a matrix Lie algebra}
As shown in Fig.~\ref{fig:lie_exp_log_relationship}, mapping from a vector to a Lie group consists of two steps: (1) mapping from the vector to the Lie algebra and (2) using the $\exp$  function to map to a matrix in the group.  The first step, commonly referred to as a \emph{hat} operator, is generally defined as an algorithmic procedure rather than a mathematical procedure in previous research literature.  In this subsection, we show how both the $hat$ and $vee$ operators can be defined as linear operators using tensor notation as follows:
\begin{equation}
    \begin{aligned}
        \vlie = \vvec^\wedge &\triangleq \vvec_i\mathcal{B}_{ijk} \\
        \vvec = \vlie^\vee &\triangleq (\bm{P}^{-1})_{ij} \mathcal{B}_{jkl} \vlie_{kl} 
    \end{aligned}
\end{equation}
where $\mathcal{B}$ is a basis set of matrices defining the Lie algebra.  The three indices to $\mathcal{B}_{ijk}$ are $i$, which indexes the matrices while $j$ and $k$ index the rows and columns of the matrix, respectively.  In this case the mapping of a vector $\vvec$ to Lie algebra entry $\vlie$ can be expressed as:
\begin{equation}
    \vlie_{jk} = \vvec_i \mathcal{B}_{ijk}
    \label{eq:defineHat}
\end{equation}

To define the $vee$ operator, we use the projection operation introduced in Section~\ref{ss:projection}, specifically equations \eqref{eq:projMatrix} and \eqref{eq:matrixProjection}, where the matrix $\bm{M}$ is replaced with the Lie algebra matrix $\vlie$.  

Note that rather than algorithmically defining the $\emph{hat}$ and $\emph{vee}$ operations, tensor notation allows us to define these operations as linear subspace synthesis and projection operations.

\subsection{Derivatives for estimation}
In addition to mathematically defining the hat and vee operators, tensors can also be used to clearly define the derivatives required for estimation when working with a matrix Lie group.  Before introducing tensor notation for derivatives, consider the traditional estimation problem where residuals ($\bm{y}$) between the measurement received ($\bm{z}$) and a predicted measurement ($\hat{\bm{z}}\triangleq h(\hat{\bm{x}})$) should be minimized, where $\bm{h}$ is a non-linear function of the state to be estimated ($\bm{x}$).  To perform gradient-based estimation, this equation is typically linearized about the current estimate $\hat{\bm{x}}$, yielding:
\begin{equation}
    \bm{y} \approx \bm{z} -(\bm{h}(\hat{\bm{x}}) + \bm{H}\Delta\bm{x})
    \label{eq:LinearizedBasicEst}
\end{equation}
where $\bm{H}\triangleq\left.\pfrac{\bm{h}(\bm{x})}{\bm{x}}\right|_{\bm{x}=\hat{\bm{x}}}$ and $\Delta\bm{x}\triangleq \bm{x}-\hat{\bm{x}}$.  To perform estimation, this $\bm{H}$ matrix is used to compute the Kalman gain, to create a least squares optimization problem for factor graph optimization, or otherwise perform a gradient-based optimization that minimizes the residual values.

In this paper, we will consider two different classes of problems that use manifolds in estimation problems, describing how the $\bm{H}$ needed for optimization can be computed for each class of problems.  These two classes are: 
\begin{enumerate}
\item \textbf{The state is an element of a matrix Lie group.}  In this case, the $\bm{x}$ that is an input to $\bm{h}(\bm{x})$ is an element of a matrix Lie group.  \emph{But}, the $\Delta \bm{x}$ from \eqref{eq:LinearizedBasicEst} is a vector representing elements of the Lie algebra of the group.
\item \textbf{The measurement is an element of a matrix Lie group.}  In this case, $\bm{h}(\bm{x})$ is a function that generates an element of a matrix Lie group from the current state ($\V$), while the measurement itself is another matrix Lie group element $\W$.  In this case, the traditional, Cartesian subtraction of the two elements will not yield an acceptable residual. Instead, the difference ($\D$) between the two elements is defined such that $\W = \V\D$, and the residual to be minimized is\footnote{Assuming right-multiplication of the group members.  For left-multiply, $\W \triangleq \D\V$, implying $\D = \W\Vi$.}:
\begin{equation}
    \bm{y} = \Log(\D) = \Log(\Vi\W).
    \label{eq:LogManifoldMeas}
\end{equation}
\end{enumerate}
We discuss these two different classes of problems in the subsections that follow.

\subsubsection{States that are elements of a matrix Lie group}
\label{sss:stateManifold}
Consider the case where a state element ($\V$) is an element of a matrix Lie group.  Assume, without loss of generality, we can write the predicted measurement function as:
\begin{equation}
    \bm{h}(\bm{x}) = \bm{g}(\bm{U}\V\bm{w})
    \label{eq:matrixStateMeas}
\end{equation}
where $\V$ is an element of a matrix Lie group and $\bm{U}$ and $\bm{w}$ are an appropriately sized matrix and vector, respectively.  Furthermore, assume that $\V$ is a function of some vector $\vvec$.  (Often, when using matrix Lie groups, $\bm{x}$ will be a member of the group (i.e., $\V$), while $\Delta\bm{x}$ will be a vector that maps to the Lie algebra (i.e. $\vvec$).  Therefore, $\pfrac{\bm{h}(\bm{x})}{\Delta\bm{x}} = \pfrac{\bm{h}}{\vvec}$.  To compute $\pfrac{\bm{h}}{\vvec}$, we first use the chain rule to combine the derivative of $\bm{g}$ with the derivative of its inputs:
\begin{equation}
    \pfrac{\bm{h}_i}{\vvec_k} = \left. \pfrac{\bm{g}_i(\bm{y})}{\bm{y}_j} \right|_{y=\bm{U}\V\bm{w}} \pfrac{(\bm{U}\V\bm{w})_j}{\vvec_k}
\end{equation}

To compute $\pfrac{(\bm{U}\V\bm{w})_j}{\vvec_k}$, two approaches can be followed.  One approach is to take an element by element derivative of $\V$ w.r.t. $\vvec$.  Note that this derivative is a 3D tensor expressed as $\pfrac{\V_{ij}}{\vvec_k}$ where $i,$ $j,$ and $k$ are tensor indices and indicate that for each element of $\vvec$, an entire matrix of partial derivatives exists\footnote{To be more consistent, we could write this tensor as $\left(\pfrac{\V}{\vvec}\right)_{ijk}$, but placing the indices in the numerator or denominator of the partial derivative makes role of each index in the tensor more explicit.}.  Using Corollary~\ref{c:singleDeriv}, this leads to an expression for $\pfrac{(\bm{U}\V\bm{w})_j}{\vvec_k}$ in Einstein notation as follows:
\begin{equation}
\pfrac{(\bm{U}\V\bm{w})_j}{\vvec_k} = \bm{U}_{jl} \pfrac{\V_{l m}}{\vvec_k} \bm{w}_m \label{eq:derivUVw}
\end{equation}
where $j,k,l$, and $m$ are tensor indices.  Note how the tensor notation clarifies the inner products that must take place for $\pfrac{\bm{h}_i}{\vvec_k}$ to be computed, overcoming the difficulty of using matrix derivatives in more complex equations.

The second approach is to use properties of the Lie group to represent local movements about $\V$.  This can be represented by replacing $\V$ with a combination of $\V$ and a ``difference'' matrix $\D$, where $\D$ represents a local movement on the manifold\footnote{Note that the difference matrix is on the \emph{right} side of $\V$.  This is an arbitrary choice and can be done on the left-side as well.  For brevity, this paper focuses on right-side multiplication.}: i.e. $\bm{U}\V\bm{w} = \bm{U}\V\bm{Dw}$. This local movement is a function of a vector $\bm{d}$ such that $\D \triangleq \Exp(\bm{d})$.  Note than when $\bm{d}$ is all zeros, $\D$ is the identity matrix and this quantity is exactly the same as $\bm{UVw}$.  Using this formulation, we replace \eqref{eq:derivUVw} with:

\begin{equation}
    \pfrac{(\bm{UVDw})_j}{\bm{d}_k} = \bm{U}_{jl}\V_{lm}\pfrac{\D_{mn}}{\bm{d}_k}\bm{w}_n
    \label{eq:derivUVDw}
\end{equation}

For matrix Lie groups, the derivative tensor $\pfrac{\D_{mn}}{\bm{d}_k}$ is surprisingly simple:
\begin{equation}
    \begin{aligned}
        \bm{D} &= \Exp(\bm{d}) && \text{by definition of }\bm{d}\text{ and }\bm{D}\\
        \bm{D} &= \exp(\mathcal{B}_{ijk}\bm{d}_i) && \text{by definition of }\Exp\text{ and }\wedge\\
        \bm{D} &\approx \bm{I} + \mathcal{B}_{ijk}\bm{d}_i && \text{for small $\bm{d}$ from \eqref{eq:exponentiation}}\\
        \pfrac{\bm{D}_{jk}}{\bm{d}_i} &= \mathcal{B}_{ijk} && \text{differentiation}
    \end{aligned}
\end{equation}
enabling Equation~\eqref{eq:derivUVDw} to be written as:
\begin{equation}
    \pfrac{(\bm{U}\V\D\bm{w})_j}{\bm{d}_k} = \bm{U}_{jl}\bm{V}_{lm}\mathcal{B}_{kmn}\bm{w}_n
\end{equation}
where $\mathcal{B}_{kmn}$ is the basis tensor for the Lie algebra

When the estimator finds a $\bm{d}$ value that it believes will minimize error in the state, it can be applied to the current state estimate using the group operator as follows:
\begin{equation}
    \V^+ = \V^-\circ \Exp(\bm{d})
\end{equation}
where $\V^-$ and $\V^+$ are the elements of the matrix Lie group before and after the optimization step, respectively. 

\subsubsection{Measurements in the matrix Lie group}
\label{sss:measManifold}
When using measurements in the matrix Lie group, one of the most difficult parts is determining the derivative of the $\Log$ function in Equation~\eqref{eq:LogManifoldMeas}. Note that the $\Log$ function has a matrix as an input, but outputs a vector the size of the tangent plane dimensionality (i.e., it is a non-injective function).  If the matrix is an element of the matrix Lie group that is differentially reachable from the identity element, the $\Log$ function is precisely defined as the inverse of $\Exp$.  However, when considering the input to be any matrix of the correct size, $\Log$ is more difficult to define.  (A point that has generally been ignored in previous papers on Lie groups.)  

With linear functions, non-injectivity can be overcome by performing a \emph{pseudo}-inverse. This operation performs a least-squares mapping to a point where the forward function is precisely defined and performs the inverse of that point.  For non-linear functions, the pseudo-inverse of first-order derivatives yields a least-squares mapping to a tangent plane about the differentiation point, and returns that point on the tangent plane.  With tensor notation, we can define a similar least-squares mapping of matrices to the linear subspace defined by the forward derivative.

Assume that (1) an initial value of $\D$ and $\bm{d}$ is known such that $\D = Exp(\bm{d})$ and (2) the derivative of $\pfrac{\D}{\bm{d}}$ is known for the current value of $\bm{d}$.  Because $\pfrac{\D}{\bm{d}}$ represents the derivative of $\Exp(\bm{d})$ w.r.t. $\bm{d}$, the pseudo-inverse yields a valid derivative of $\Log$, with an implied least squares mapping from any matrix to the tangent plane about $\D$.  Utilizing the projection operation defined in Section~\ref{ss:projection}, we compute the pseudo inverse of $\pfrac{\D}{\bm{d}}$ as:
\begin{equation}
    \begin{aligned}
    \pfrac{\Log(\D)_i}{\D_{kl}} \triangleq (\bm{P}^{-1})_{ij} \pfrac{\D_{kl}}{\bm{d}_j} && \text{where}\\
    \bm{P}_{mn} \triangleq \pfrac{\D_{op}}{\bm{d}_m}\pfrac{\D_{op}}{\bm{d}_n}
    \end{aligned}
    \label{eq:projectionTensor}
\end{equation}

Computing the derivative of $\Log$ using the pseudo-inverse has several advantages.  First, note that for arbitrary $\bm{d}$, the derivative of $\Log$ may be difficult to compute.  For example, the SO3 (rotation matrix) group has closed-form expressions for the derivative of $\Exp$, but there is no single, closed-form expression for the derivative of $\Log$ that works for all possible rotations.  Using the inverse of the derivative of $\Exp$ side-steps this issue.  Second, even when the derivative of $\Log$ does have a closed-form expression, it is generally just for matrices that are in the Lie group.  Using a projection operator broadens the derivative to accept any perturbation in the matrix in a mathematically sound way.  Note that this method of computing the derivative of $\Log(\D)$ is novel and not presented previously in the research literature.  

\section{Results/Demonstration}
\label{s:results}
To demonstrate the utility of tensor notation when optimizing over Matrix Lie groups, we present results of two optimization-based estimation scenarios.  Both scenarios implement a Gauss Newton optimization procedure as outlined in Algorithm~\ref{alg:GN}.  Note that both the measurements ($\mathbb{Y}$) and the states being estimated ($\mathbb{X}$) consist of an ordered set of elements.  In the first example, the measurements are Cartesian vectors and the state is a single SO3 matrix.  In the second example, the measurements are SO3 matrices, while the states are also a set of SO3 matrices.  The residual and the $\Delta \bm{x}$ are vectors however, enabling a Jacobian matrix to be computed (line 4) to find the next step in the optimization routine.  More details for both examples are described in the following sub-sections.

Note that the matrix Lie group used in both estimation problems is SO3.  This requires a library implementing the $\Exp$ and $\Log$ functions, together with a $\pfrac{\Exp(\vvec)}{\vvec}$ function.  Details of how these were implemented for SO3 are given in Appendix~\ref{a:so3}. Python code implementing both estimation problems can be found at \url{https://github.com/cntaylor/tensor_matrix_lie_group_example/}.  

\begin{algorithm}
\caption{Gauss-Newton Optimization}
\label{alg:GN}
\begin{algorithmic}[1]
\renewcommand{\algorithmicrequire}{\textbf{Input:}}
\renewcommand{\algorithmicensure}{\textbf{Output:}}
\REQUIRE Measurements $\mathbb{Y}$, Initial state estimate $\mathbb{X}_0$
\ENSURE Optimized state $\mathbb{X}^*$

\STATE $\mathbb{X} \leftarrow \mathbb{X}_0$

\WHILE{not converged}
    \STATE Compute residual vector: $\bm{r}(\mathbb{X},\mathbb{Y})$
    \STATE Compute Jacobian matrix: $\bm{L} = \pfrac{\bm{r}}{\Delta\bm{x}}$
    \STATE Solve normal equations for step $\Delta \mathbf{x}$:\\  $\Delta\bm{x} = -(\bm{L}^\top\bm{L})^{-1}\bm{L}^\top\bm{r}$
    \STATE Update state estimate: $\mathbb{X} \leftarrow f(\mathbb{X}, \Delta \bm{x})$
\ENDWHILE
\RETURN $\mathbb{X}^* = \mathbb{X}$
\end{algorithmic}
\end{algorithm}

\subsection{Estimation with state space matrix Lie groups }
In the first scenario, the goal is to estimate an SO3 matrix that best aligns two sets of associated 3D points.  Therefore, the state estimate ($\mathbb{X}$) is a single rotation (SO3) matrix and the residuals are a function of ten 3D vectors ($\bm{u}_1, \bm{u}_2,...,\bm{u}_{10}$) and their rotated versions ($\bm{v}_1,...,\bm{v}_{10}$).  The ten original vectors are randomly sampled from within a zero-centered cube with side lengths of 20, the true rotation matrix $\bm{R}$ is also randomly chosen, and the rotated points are computed as $\bm{v}_i=\bm{Ru}_i$.  Zero-mean, Gaussian noise with a standard deviation of 0.1 is then added to all $\bm{u}$ and $\bm{v}$ vectors, forming the set of measurements $\mathbb{Y}$.

To perform the Gauss Newton optimization as outlined in Algorithm~\ref{alg:GN}, the following quantities are defined:

\begin{align*}
\mathbb{X}_0 &= \{\bm{X}\} = {\bm{I}} && \text{\parbox{4cm}{Initial rotation estimate is the identity matrix}}\\
\bm{r}(\mathbb{X},\mathbb{Y}) &= \begin{bmatrix}\bm{Xu}_1 - \bm{v}_1\\
\vdots\\\bm{Xu}_{10} - \bm{v}_{10}\end{bmatrix} && \text{\parbox{4cm}{Note that $\bm{X}$ is a rotation (SO3) matrix}}\\
\bm{L} &= \begin{bmatrix}
    \bm{X}_{ij}\mathcal{B}_{\ell jk}(\bm{u}_1)_k\\
    \vdots\\
    \bm{X}_{ij}\mathcal{B}_{\ell jk}(\bm{u}_{10})_k\\
\end{bmatrix} && \text{\parbox{4cm}{$\bm{L}$ is a $30\times 3$ matrix. Each tensor equation yields a 2D ($3\times 3$) tensor with indices $i\ell$.  $i$ represents rows and $\ell$ the columns}}\\
f(\mathbb{X},\Delta\bm{x}) &= \bm{X}\Exp(\Delta \bm{x}) && \text{\parbox{4cm}{Update $\bm{X}$ using matrix multiplication}}
\end{align*}
 
Note how using the tensor notation yields explicit and compact equations for the derivative of the residual function w.r.t. the rotation matrix being estimated.  The rotation estimates obtained using this procedure are also highly accurate, with a representative result shown in Fig.~\ref{fig:stateEstResults}.  

\begin{figure}[hb]
\begin{center}\begin{tabular}{cc}
\textbf{Truth} & \textbf{Estimated Value}\\
$\begin{bmatrix}
0.911 & -0.346 & 0.223 \\
0.294 & 0.926 & 0.238 \\
-0.289 & -0.152 & 0.945
\end{bmatrix}$ & 
$\begin{bmatrix}
0.913 & -0.338 & 0.228 \\
0.284 & 0.928 & 0.240 \\
-0.293 & -0.154 & 0.944
\end{bmatrix}$
\end{tabular}
\end{center}
    \caption{Representative results from optimization using tensor derivatives on a matrix Lie group.}
    \label{fig:stateEstResults}
\end{figure}


\subsection{Estimation with measurements on a matrix Lie group}
Now let us consider a more complex example that (a) has measurements that are in a matrix Lie group and (b) is estimating multiple SO3 matrices simultaneously.  The motivation for this scenario is a satellite attitude estimation problem, where a star tracker gives rotation as a measurement, and gyroscopes give relative rotation between time steps.  A factor graph representation of this problem is shown in Fig.~\ref{fig:manifoldMeasGraph}, where each individual variable ($\bm{X}_i$) to be estimated is an SO3 matrix.  The individual measurements (the black boxes) at each state are SO3 matrices, and the green dynamics are also SO3 matrices.  Assuming the true rotations are written as $\bm{R}_i$, then the gyroscope measurement equations are:
\begin{equation}
\bm{g}_i = \Log(\bm{R}_i^\top\bm{R}_{i+1}) + \nu_g
\end{equation}
with $\nu_g \in \mathcal{N}(0,\sigma_g^2)$.  The absolute rotation measurement for each timestep is given by:
\begin{equation}
    \bm{M}_i = \Exp\left(\Log(\bm{R}_i)+\nu_m\right)
\end{equation}
with $\nu_m \in \mathcal{N}(0,\sigma_m^2)$.  In this simulation, $\sigma_m =$1E-1 and $\sigma_g=$1E-4. 

\begin{figure*}
    \centering
    \includegraphics[width=0.8\textwidth]{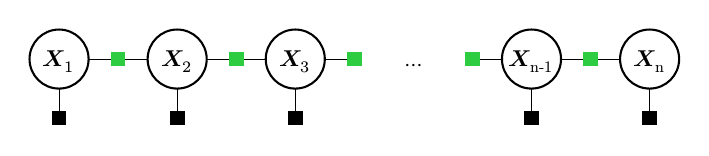}
    \caption{Factor graph used to demonstrate \emph{measurements} that are elements of a matrix Lie group.}
    \label{fig:manifoldMeasGraph}
\end{figure*}

To perform the Gauss-Newton optimization outlined in Algorithm~\ref{alg:GN}, the quantities defined in Fig.~\ref{fig:measOptVals} are used,
\begin{figure*}
\begin{equation*}
    \begin{aligned}
    \mathbb{X}_0 &= \{\bm{M}_1, \bm{M}_2, ..., \bm{M}_n\} &&\text{Initialize the states with the measurements}\\
    \bm{r}(\mathbb{X},\mathbb{Y}) &= \begin{bmatrix}
        \Log(\bm{M}_1^\top\bm{X}_1)\\
        \vdots\\
        \Log(\bm{M}_n^\top\bm{X}_n)\\
        \Log(\bm{X}_{1}^\top\bm{X}_{2}) - \bm{g}_{1}\\
        \vdots\\
        \Log(\bm{X}_{n-1}^\top\bm{X}_{n}) - \bm{g}_{n-1}
    \end{bmatrix}&& \text{\parbox{4.5cm}{Stacked residuals.  A $3(2n-1)$ element vector}}\\
    \bm{L} &= \begin{bmatrix}
        \bm{J}_1 & \bm{0} & \bm{0} & \cdots & \bm{0}\\
        \bm{0} & \bm{J}_2 & \bm{0} & \cdots & \bm{0}\\
        \vdots & \vdots & \vdots & \ddots & \vdots\\
        \bm{0} & \bm{0} & \bm{0} &  \cdots & \bm{J}_n\\
        \bm{G}_1 & \bm{H}_1 & \bm{0} & \cdots & \bm{0}\\
        \bm{0} & \bm{G}_2 & \bm{H}_2 & \cdots & \bm{0}\\
        \bm{0} & \bm{0} & \bm{0} & \cdots & \bm{H}_{n-1}
    \end{bmatrix}&&\text{a 3(2n-1)$\times$3n matrix}\\
    f(\mathbb{X},\Delta\bm{x}) &= \left[\bm{X}_i\Exp(\Delta\bm{x}_i)\right]_i&&\text{Update each rotation state individually}
\end{aligned}
\end{equation*}
\caption{Quantities needed to perform second optimization example}
\label{fig:measOptVals}
\end{figure*}
where $\bm{L}$ is a block, sparse matrix, $\bm{J}_i$ represents the derivative of $\Log(\bm{M}_i^\top\bm{X}_i)$ w.r.t. $\bm{X}_i$, and $\bm{G}_i$ and $\bm{H}_i$ represent the derivative of $\Log(\bm{X}_i^\top\bm{X}_{i+1})$ w.r.t. $\bm{X}_i$ and $\bm{X}_{i+1}$, respectively.  The derivation for these quantities using the tensor notation and projection operation are defined below.

To define $\bm{J}$, $\bm{G}$, and $\bm{H}$, we first define a function that takes in a 3-vector and generates a 3D tensor representing the $\Log$ derivatives using Equation~\eqref{eq:projectionTensor}:
\begin{equation}
    \begin{aligned}
        \mathcal{P}(\bm{d})_{ikl} &\leftarrow (\bm{P}^{-1})_{ij} \mathcal{D}_{jkl} & \text{where}\\
        \mathcal{D}_{ijk} &\triangleq \pfrac{\Exp(\bm{d})_{jk}}{\bm{d}_i} \\
        \bm{P}_{ij} &= \mathcal{D}_{ikl} \mathcal{D}_{jkl}
    \end{aligned}
\end{equation}
Using the chain rule, this function output is combined with the derivative of the items inside the $\Log$ function to compute the derivatives needed for Gauss-Newton optimization.  The three derivatives (one for the absolute measurements, two for the gyroscopes) needed are\footnote{In the 2nd equation, note that the transpose is performed using tensor indices.}:
\begin{align*}
    \pfrac{(\bm{M}^\top\bm{X})_{jk}}{\Delta \bm{x}_i} &= \bm{M}_{lj}\bm{X}_{lm}\mathcal{B}_{imk}\\
    \pfrac{(\bm{X}_\alpha^\top\bm{X}_{\alpha+1})_{jk}}{(\Delta\bm{x}_\alpha)_i} &= \mathcal{B}_{ilj}(\bm{X}_\alpha)_{ml}(\bm{X}_{\alpha+1})_{mk}\\
    \pfrac{(\bm{X}_\alpha^\top\bm{X}_{\alpha+1})_{jk}}{(\Delta\bm{x}_{\alpha+1})_i} &= (\bm{X}_\alpha)_{lj}(\bm{X}_{\alpha+1})_{lm}\mathcal{B}_{imk}
\end{align*}

These are combined to obtain the final formulas:
\begin{equation}
    \begin{aligned}
        \bm{J}_{i\ell} &= \mathcal{P}\left(\Log(\bm{M}^\top\bm{X})\right)_{ijk}\pfrac{(\bm{M}^\top\bm{X})_{jk}}{\Delta\bm{x}_\ell}\\
        \bm{G}_{i\ell} &= \mathcal{P}\left(\Log(\bm{X}_\alpha^\top\bm{X}_{\alpha+1})\right)_{ijk}\pfrac{(\bm{X}_\alpha^\top\bm{X}_{\alpha+1})_{jk}}{(\Delta\bm{x}_{\alpha})_\ell}\\
        \bm{H}_{i\ell} &= \mathcal{P}\left(\Log(\bm{X}_\alpha^\top\bm{X}_{\alpha+1})\right)_{ijk}\pfrac{(\bm{X}_\alpha^\top\bm{X}_{\alpha+1})_{jk}}{(\Delta\bm{x}_{\alpha+1})_\ell}\\
    \end{aligned}
\end{equation}

Results for performing optimization with this factor graph (where all measurements are on the matrix Lie group) are shown in Table~\ref{tab:measMLG_results}.  Because the standard deviation of the dynamic (between factors) is small compared to the measurement covariance, we would expect the overall accuracy to $\approx\frac{\sigma_m}{\sqrt{n}}$ where $n$ is the number of factors in the graph.  In Table~\ref{tab:measMLG_results}, we show the standard deviation across 1000 runs of the factor graph, where each run had a different (randomly selected) set of truth values, and corresponding measurements and dynamics.  Note that the achieved values closely match the theoretically predicted values. This demonstrates the capability of the tensor-based notation to properly handle the derivative of the $\Log$ functions using the technique described in Section~\ref{sss:measManifold}.

\begin{table}[]
    \centering
    \caption{Standard deviation of factor graph optimized values for differing number of nodes.  Note that the achieved performance closely matches the theoretically predicted performance.}
    \label{tab:measMLG_results}
    \begin{tabular}{|c|c|c|c|}
    \hline
         Number of states (n) & 5 & 10 & 20  \\\hline
         Achieved accuracy $\sigma$ & 4.49E-2 &  3.25E-2 & 2.26E-2\\\hline
         Theoretical accuracy $\sigma=\frac{1E-1}{\sqrt{n}}$ & 4.47E-2 & 3.16E-2 & 2.24E-2\\\hline
    \end{tabular}
\end{table}

\section{Conclusion}
\label{s:concl}
In this paper, we have introduced a novel method for representing the matrix Lie group derivatives required for estimation using tensors and Einstein summation notation.  We believe this representational technique is both more explicit and easier to understand than previous methods used to describe these derivatives.  In addition, using tensor notation, we can define a matrix projection operation that simplifies the computation and derivatives of the $\Log$ function required for matrix Lie group operations.  The efficacy of this technique has been demonstrated using optimization on rotation matrices.

\bibliographystyle{plain}
\bibliography{main}

\appendix
\section{SO3 library implementation}
\label{a:so3}
To implement optimization using a matrix Lie group, three functions are required: $\Exp$, $\Log$, and $\pfrac{\V}{\vvec}$.  How these functions were implemented for the special orthogonal group of $3\times 3$ matrices (SO3) is outlined in this appendix.  Note that the Lie algebra for SO3 matrices is the set of skew symmetric matrices, yielding the basis tensor:
\begin{equation}
\begin{aligned}
    \mathcal{B}_{1,jk} &= \begin{bmatrix} 0 & 0 & 0 \\ 0 & 0 & -1 \\ 0 & 1 & 0 \end{bmatrix} \\
    \mathcal{B}_{2,jk} &= \begin{bmatrix} 0 & 0 & 1 \\ 0 & 0 & 0 \\ -1 & 0 & 0 \end{bmatrix}  \\
    \mathcal{B}_{3,jk} &= \begin{bmatrix} 0 & -1 & 0 \\ 1 & 0 & 0 \\ 0 & 0 & 0 \end{bmatrix} &
\end{aligned}
\end{equation}  
The computations are here for completeness, but a more detailed explanation can be found in \cite{blanco2021tutorial}.

\subsection{$\Exp$}
Using a Taylor Series expansion of the exponential operator ($\mathrm{exp}$) and the properties of skew symmetric matrices, we obtain:
\begin{equation}
\begin{aligned}
    \bm{S} &= \mathcal{B}_{ijk}\vvec_i && \parbox{4cm}{\raggedright The ``hat'' operator yields a skew symmetric matrix}\\
    \theta &= ||\vvec|| && \text{Defined for compactness}\\
    \V &= \Exp(\vvec) && \text{The desired expression}\\
       &= \mathrm{exp}(\bm{S}) && \text{By definition of }\Exp\\
       &= \bm{I} + \frac{\sin\theta}{\theta}\bm{S} + \frac{1-\cos\theta}{\theta^2}\bm{S}^2 && \parbox{4cm}{By Taylor series expansion of $\exp$ and skew symmetric matrix properties}
\end{aligned}
\label{eq:deriveExpSO3}
\end{equation}
For completeness, this yields the expression in Equation~\eqref{eq:fullExpSO3}.
\begin{figure*}[h!t!]
\begin{equation}
    \Exp(\vvec) = \begin{bmatrix}
        1 - \frac{1-\cos\theta}{\theta^2}(v_2^2+v_3^2) & -\frac{\sin\theta}{\theta}v_3 + \frac{1-\cos\theta}{\theta^2}v_1v_2 & \frac{\sin\theta}{\theta}v_2 + \frac{1-\cos\theta}{\theta^2}v_1v_3 \\
        \frac{\sin\theta}{\theta}v_3 + \frac{1-\cos\theta}{\theta^2}v_1v_2 & 1 - \frac{1-\cos\theta}{\theta^2}(v_1^2 + v_3^2) & -\frac{\sin\theta}{\theta}v_1 + \frac{1-\cos\theta}{\theta^2}v_2v_3 \\
        -\frac{\sin\theta}{\theta}v_2 + \frac{1-\cos\theta}{\theta^2}v_1v_3 & \frac{\sin\theta}{\theta}v_1 + \frac{1-\cos\theta}{\theta^2}v_2v_3 & 1 - \frac{1-\cos\theta}{\theta^2}(v_1^2 + v_2^2)
    \end{bmatrix}
    \label{eq:fullExpSO3}
\end{equation}
\end{figure*}

\subsection{$\Log$}
To compute the $\Log$ of a SO3 matrix ($\vvec = \Log(\V)$), two properties of the matrix in \eqref{eq:fullExpSO3} should be noted.  First, note that the basis tensor can be used to operate on this matrix to yield a scaled version of $\vvec$.
\begin{equation}
    2\frac{\sin\theta}{\theta}\vvec = \mathcal{B}_{ijk}\V_{jk}
\end{equation}
Second, to find $\theta$, note that the trace of the matrix yields $3-2(1-\cos\theta)$ -- using the fact that $\frac{v_1^2 + v_2^2 + v_3^2}{\theta^2}=1$.  Therefore,
\begin{equation}
    \vvec = \frac{\theta}{2\sin\theta}\mathcal{B}_{ijk}\V_{jk}
\end{equation}
While this method works well for most values of $\theta$, as $\theta \rightarrow \pi$, the scaling term $\frac{\theta}{\sin \theta}$ approaches infinity. Therefore an alternative method for computing log is to compute a vector ($\bm{s}$) using just the diagonal values of $\V$.
\begin{equation}
\begin{aligned}
    \bm{s} &= \frac{2 diag(\V) + \bm{1}(1-tr(\V))}{3 - tr(\V)}\\
    &= \frac{1}{\theta^2}\begin{bmatrix}
        \vvec_1^2\\
        \vvec_2^2\\
        \vvec_3^2
    \end{bmatrix}
\end{aligned}
\end{equation}
where $diag(\V)$ is a vector of the diagonal elements in $\V$, $tr(\V)$ is the trace of $\V$ and $\bm{1}$ is a 3-element vector of ones. Taking the square root of each element in $\bm{s}$ and multiplying by $\theta$ yields the absolute value of each element in $\vvec$.  Using the sign of the elements in  $\mathcal{B}_{ijk}\V_{jk}$ completes the $\Log$ computation.

To determine which method to use, the value of the trace of $\V$ is used.  When the value is closer to -1, the alternative method is used as it is more numerically stable.  By combining these two methods together, a numerically stable $\Log$ function is derived for all values of $\theta$.  Unfortunately, the derivative of both of these techniques is numerically unstable as $\theta\rightarrow\pi$.  


\subsection{\texorpdfstring{\boldsymbol$\pfrac{\Exp(\vvec)}{\vvec}$}{dV/dv}}
To compute the derivative of $\V=\Exp(\vvec)$ as defined in Equation~\eqref{eq:deriveExpSO3} w.r.t. $\vvec$, we will define the derivative of $\theta$, $\bm{S}$, and $\bm{S}^2$ w.r.t. $\vvec$ and then combine these results with the chain and product rules:
\begin{align}
        \notag\theta &= \sqrt{\vvec_i\vvec_i}\\\notag
        \pfrac{\theta}{\bm{v}_i} &= \frac{\bm{v}_i}{\theta}\\\notag
        \bm{S}_{jk} &= \vvec_i\mathcal{B}_{ijk}\\\notag
        \pfrac{\bm{S}_{jk}}{\bm{v}_i} &= \mathcal{B}_{ijk}\\
        (\bm{S}^2)_{jl} &=  \bm{S}_{jk}\bm{S}_{kl} = (\bm{v}_i\mathcal{B}_{ijk})(\bm{v}_i\mathcal{B}_{ikl})\\\notag
        \pfrac{\bm{S}^2_{jl}}{\bm{v}_i} &= \mathcal{B}_{ijk}\bm{S}_{kl} + \bm{S}_{jk}\mathcal{B}_{ikl}\\\notag
        \pfrac{exp(\bm{S})}{\theta} &= \left(\frac{\cos\theta}{\theta} - \frac{\sin\theta}{\theta^2}\right)\bm{S} + \left(\frac{\sin\theta}{\theta^2} -2\frac{1-\cos\theta}{\theta^3}\right)\bm{S}^2\\\notag
        \pfrac{\V_{jk}}{\vvec_i} &= \pfrac{exp(\bm{S})_{jk}}{\theta}\pfrac{\theta}{\bm{v}_i} + \frac{\sin\theta}{\theta}\pfrac{\bm{S}_{jk}}{\bm{v}_i} + \frac{1-\cos\theta}{\theta^2}\pfrac{\bm{S}_{jk}^2}{\bm{v}_i}
\end{align}

\onecolumn
\center{\Large Graphical Abstract for Review}
\\[1cm]

\includegraphics[width=\textwidth]{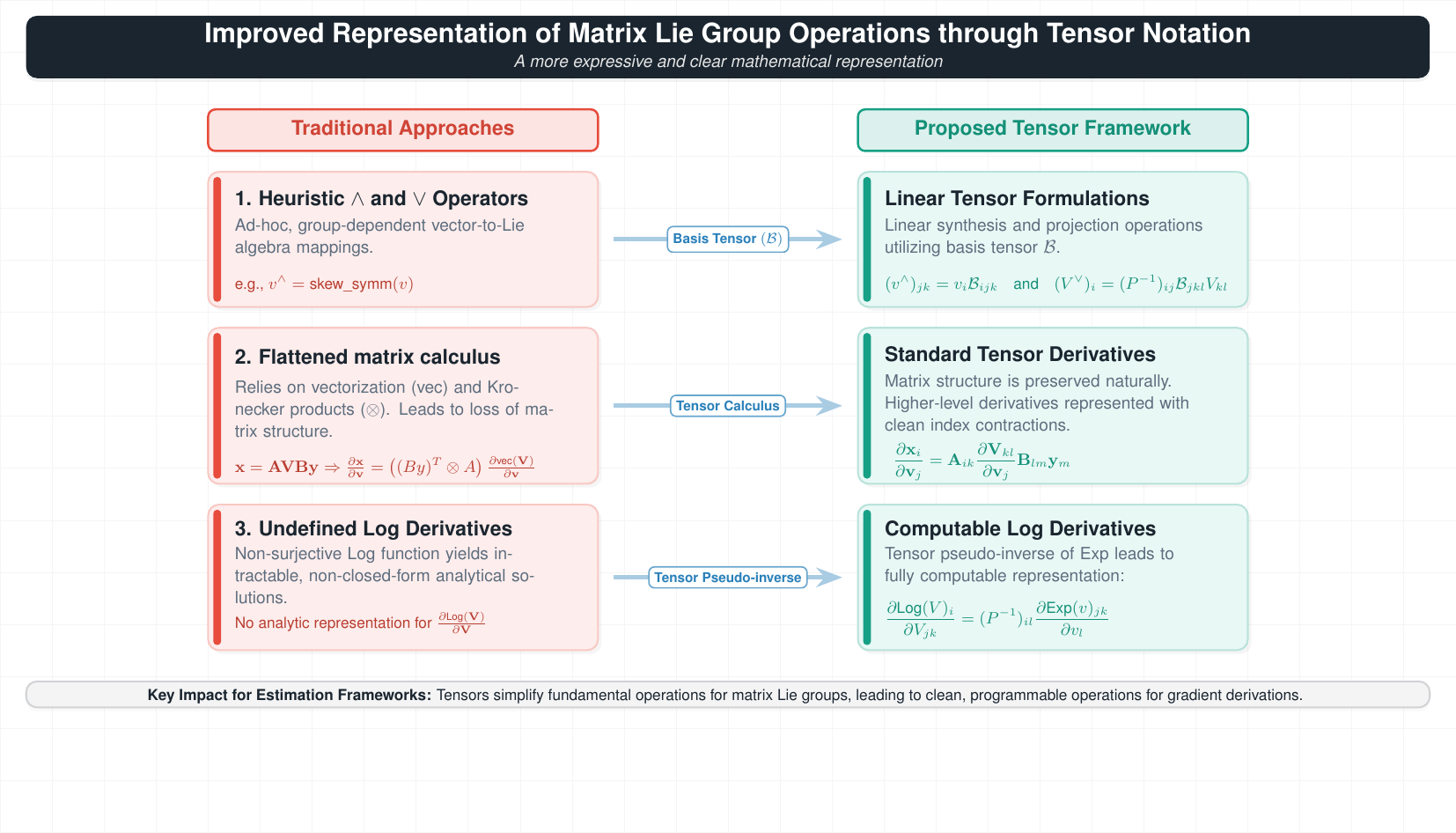}
\end{document}